\documentclass[10pt,twocolumn,letterpaper]{article}

\usepackage{iccv}
\usepackage{times}
\usepackage{epsfig}
\usepackage{graphicx}
\usepackage{amsmath}
\usepackage{amssymb}
\usepackage{multirow}
\usepackage{amsthm}
\usepackage{subfigure}
\usepackage{epstopdf}

\newcommand\xuworries[1]{\textcolor{black}{#1}}

\newcommand{\bbS}{\mathbb{S}}

\newcommand{\bx}{\mathbf{x}}

\newcommand{\cX}{\mathcal{X}}

\newcommand{\bp}{\mathbf{p}}

\newcommand{\bbR}{\mathbb{R}}
\newcommand{\bbE}{\mathbb{E}}
\newtheorem{prop}{Proposition}

\usepackage[pagebackref=true,breaklinks=true,letterpaper=true,colorlinks,bookmarks=false]{hyperref}
\pdfoutput=1

\iccvfinalcopy 

\def\iccvPaperID{2100} 

\ificcvfinal
\fi
\begin{document}

\title{Learning Spread-out Local Feature Descriptors}

\author{Xu Zhang$^1$, Felix X. Yu$^2$, Sanjiv Kumar$^2$, Shih-Fu Chang$^1$ \\
$^1$Columbia University, $^2$ Google Research\\
{\tt\small\{xu.zhang, sc250\}@columbia.edu, \{felixyu, sanjivk\}@google.com}}

\maketitle



\begin{abstract}
We propose a simple, yet powerful regularization technique that can be used to significantly improve both the pairwise and triplet losses in learning local feature descriptors.
The idea is that in order to fully utilize the expressive power of the descriptor space, good local feature descriptors should be sufficiently ``spread-out'' over the space.
In this work, we propose a regularization term to maximize the spread in feature descriptor inspired by the property of uniform distribution. 
We show that the proposed regularization with triplet loss outperforms existing Euclidean distance based descriptor learning techniques by a large margin. 
As an extension, the proposed regularization technique can also be used to improve image-level deep feature embedding.
\end{abstract}

\section{Introduction}
Computing image patch correspondences based on local descriptor matching is important in many computer vision problems such as image retrieval, wide baseline stereo matching and panorama building. 
The main challenge of finding correct correspondences is that the appearance of the image patches varies due to changes of scaling, view angle, illumination and imaging condition etc. 
Designing local feature descriptors that are invariant to such changes is therefore essential. Efforts of local descriptor fall into two categories: hand-crafted and learning-based.
Hand-crafted descriptors try to achieve the invariance by manually selected rules. 
One of the most popular hand-crafted descriptors is SIFT~\cite{lowe_distinctive_2004} and its variants ~\cite{bay2006surf,tola2010daisy}, which are widely used in the computer vision community. 
The main issue of the hand-crafted descriptors is that they can only consider a limited predefined set of variations.

One approach to take all variations into consideration is learning local descriptors from a large patch correspondence dataset~\cite{brown_discriminative_2011,simonyan_learning_2014}. 
The state-of-the-art descriptor learning methods are based on neural networks \cite{balntas_learning_2016,g_learning_2015,simo-serra_discriminative_2015,yi_lift:_2016}. 
In addition to the model itself, the most important aspect of learning-based method is the loss function which defines the goal of descriptor learning: matching patches 
should be close in the descriptor space, while the non-matching patches should be far-away\footnote{We use Euclidean distance in this paper. See Section \ref{sec:back} for more details, and Section \ref{sec:method} for the discussion on alternatives.}. 
The pairwise loss and triplet loss (Section \ref{sec:back}) are the commonly used loss functions to achieve the desired properties. 
Recently, there are a lot of works such as smart sampling strategies~\cite{balntas_learning_2016,mishchuk2017working} and structured loss~\cite{tian2017l2} that improve the triplet loss. 
In particular, Kumar \etal~\cite{g_learning_2015} propose to use a global loss to separate the distance distributions of the matching pairs and non-matching pairs. This approach avoids the design of complicated sampling strategies and is also shown to provide results that are robust to training with outliers. 

The success of global loss motivates us to further explore the desired properties of the descriptor space and design a robust regularization term based on these properties.
Our main idea is that the good local feature descriptors should be sufficiently ``spread-out'' in the descriptor space in order to fully utilize the expressive power of the space. 
Specifically, we introduce a regularization term that induces the spread-out condition, inspired by the properties of the uniform distribution on unit sphere (Section \ref{sec:method}). 
The regularization can be easily used to improve all methods where pairwise or triplet loss is used.
We show that the proposed regularization with triplet loss, without hard sample mining, outperforms all the Euclidean distance based descriptors by a large margin (Section \ref{sec:local}). In particular, it outperforms the global loss~\cite{g_learning_2015} in the patch pair classification task. 
As an extension of descriptor learning, we show that the proposed regularization can also be used in improving image-level deep feature embedding (Section \ref{sec:image}).

\section{Background}
\label{sec:back}
We begin by reviewing some commonly used loss functions in learning local feature descriptors. 
Let $\cX = \{\bx_1, \ldots, \bx_N\}, ~ \bx_i \in \bbR^{m \times n}$ denote a set of $N$ training patches with $ m\times n$ pixels. $\{y_{ij}, 1 \leq i, j \leq N\}$ is a set of pairwise labels for $\cX$ indicating whether $\bx_i$ and $\bx_j$ belong to the same class ($y_{ij} = 1$) or not ($y_{ij} = 0$). 
In this paper, we call the pairs with $y_{ij} = 1$ the matching pairs, and the pairs with $y_{ij} = 0$ the non-matching pairs.
The goal of descriptor learning is to learn a feature embedding $f(\cdot): \bbR^{m \times n} \mapsto \bbR^d$ that maps raw patch pixels to a $d$ dimensional vector, such that $\parallel f(\bx_i) - f(\bx_j) \parallel_2$ is small when $y_{ij} = 1$ and $\parallel f(\bx_i) - f(\bx_j) \parallel_2$  is large when $y_{ij} = 0$. In this paper, we assume that $f(\cdot)$ lives on the unit sphere, i.e., $\parallel f(\bx) \parallel_2 = 1$, $\forall \bx \in \bbR^{m \times n}$.

\subsection{Pairwise loss}
\label{subsec:pair}
The pairwise loss tries to directly induce small distance for matching pairs and large distance for non-matching pairs.
An input for pairwise loss is of the form $(\bx_i,\bx_j,y_{ij})$, consisting of a pair of samples and their corresponding label. 
The most widely used pairwise loss is the contrastive loss:
\begin{equation}
\setlength\abovedisplayskip{5pt}
\setlength\belowdisplayskip{5pt}
\begin{aligned}
	\ell_{\text{con}} & = y_{ij}\max(0,\parallel f(\bx_i)-f(\bx_j)\parallel_2-\epsilon^+)\\
	& + (1-y_{ij})\max(0,\epsilon^- - \parallel f(\bx_i)-f(\bx_j)\parallel_2),
\end{aligned}
\label{eq:pairwise}
\end{equation}
where $f(\cdot)$ is the feature embedding. $\epsilon^+$ and $\epsilon^-$ control the margins of the matching and non-matching pairs respectively.
%
Contrastive loss was originally proposed in \cite{chopra2005learning} with $\epsilon^+ = 0$. As shown in \cite{lin2015deephash}, this often leads to overfitting. And a proper relaxed margin ($\epsilon^+>0$) can achieve better performance. The main problem with the pairwise loss is that the margin parameters are often difficult to choose~\cite{ustinova_learning_2016}.

\subsection{Triplet loss}
\label{subsec:triplet}
Triplet loss takes a triplet of samples as input. One triplet consists of three samples: $(\bx_i,\bx_j,\bx_k)$, with $y_{ij} = 1$ and $y_{ik} = 0$. 
To simplify the notation, we denote one triplet as $(\bx_i,\bx^+_i,\bx^-_i)$, where $\bx^+_i = \bx_j$ and $\bx^-_i = \bx_k$.
One commonly used triplet loss is the ranking loss~\cite{schultz2003learning}:
\begin{equation}
\setlength\abovedisplayskip{5pt}
\setlength\belowdisplayskip{5pt}
\begin{aligned}
	\ell_{\text{tri}} & =\max\Big(0, \epsilon - (\parallel f(\bx_i)-f(\bx^-_i) \parallel_2 \\
	& - \parallel f(\bx_i)-f(\bx^+_i)\parallel_2)\Big)
\end{aligned}
\label{eq:triplet}
\end{equation}
where $\epsilon$ is a margin. 
The idea of ranking loss is to separate the matching sample and the non-matching sample by at least a margin $\epsilon$. 
The main difference between pairwise loss and triplet loss is that pairwise loss considers the absolute distances of the matching pairs and non-matching pairs, while triplet loss considers the relative difference of the distances between matching and non-matching pairs.
Since the quality of the embeddings largely depends on the relative ordering of the matching pairs and non-matching pairs, triplet loss shows better performance than pairwise loss in local descriptor learning~\cite{balntas_learning_2016,g_learning_2015}.

\subsection{Improvements}
\label{subsec:improve}
The main issue of triplet loss and pairwise loss is that as the number of training samples grows, sampling all the possible triplets/pairs becomes infeasible, and only a relatively small portion of triplets/pairs can be used in training. As observed in practice, the training is often ineffective since many of the sampled triplets/pairs will satisfy the constraint within just a few training steps. One possible solution is to remove the ``easy samples'' and add new ``hard samples'' to the training set. However, determining which samples to remove or add is a challenging task~\cite{schroff2015facenet,simo-serra_discriminative_2015}. Additionally, focusing only on the samples that violate the training constraints the most will lead to overfitting~\cite{schroff2015facenet}. 

Balntas \etal~\cite{balntas_learning_2016} propose an improved version of triplet loss by applying in-triplet hard negative mining. 
The idea is that one triplet contains two non-matching pairs $(\bx_i,\bx^-_i)$ and $(\bx^+_i,\bx^-_i)$, and choosing the one that violates the triplet constraint more will make training more effective. They call this technique ``anchor swap''.

Kumar \etal~\cite{g_learning_2015} propose a global loss and combine it with the traditional triplet loss to address the sampling issue in pairwise and triplet loss.
Instead of considering sample pair or triplet, the global loss considers all matching and non-matching pairs in one training batch, and calculate the empirical mean and variance of the distance of the matching and non-matching pairs. The main idea of the global loss is to separate two empirical means by a margin and minimize the variances. 
There are two drawbacks of this method. First, the distribution of the distance of the matching pairs can vary greatly across different classes, and using a batch of randomly sampled matching pairs to estimate that distribution is unstable. Second, the extra margin in global loss adds extra complexity for training. 

\xuworries{Structured loss \cite{mishchuk2017working,movshovitz-attias_no_2017,oh2016deep,sohn_improved_2016,tian2017l2} considers all the possible matching and non-matching pairs in one batch of samples. By carefully designing the loss functions, structured loss has the ability to focus on the ``hard'' pairs in training. Song \etal \cite{oh2016deep} propose the lifted structured similarity softmax loss (LSSS). 
%
N-pair loss~\cite{sohn_improved_2016} further develops the idea by using a more effective batch construction method.} 

\xuworries{The motivation of this paper is that good descriptors should fully utilize the expressive power of the whole space (``spread-out'' in the descriptor space). We also propose a simple regularization term, global orthogonal regularization, to encourage the ``spread-out'' property. 
The global orthogonal regularization can be easily incorporated into other losses. Experiments show that the proposed regularization can improve the performance of different types of losses, especially those originally without the ``spread-out'' property.}

\section{Methodology}
\label{sec:method}
\subsection{``Spread-out'' local descriptors}
The main idea of this paper is that in order for the descriptors to fully utilize the descriptor space, it should be sufficiently ``spread-out'' in the descriptor space.
On the contrary, suppose there is part of the space where no feature descriptor appears, the learned feature descriptor is not fully utilizing the expression power of the space. 

One intuitive way to characterize ``spread-out'' is that: 
\emph{Given a dataset, we say that the learned descriptors are spread-out if two randomly sampled non-matching descriptors are close to orthogonal with a high probability.}
As an obvious example, we notice that uniform distribution has such property.
\begin{figure}[t]
	\begin{center}
	\includegraphics[width=0.5\linewidth]{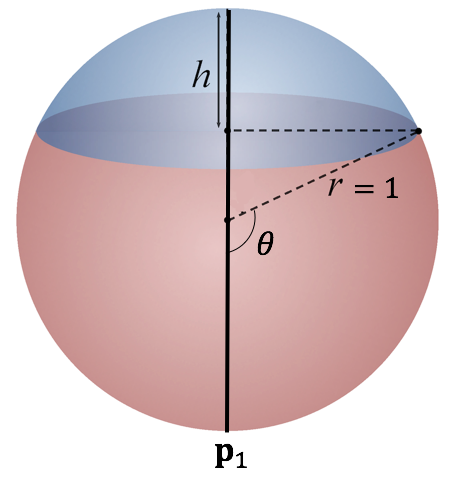}
	\end{center}
	\caption{Valid area for $\bp_1^T\bp_2 \leq s=\cos\theta$. If $\bp_2$ is on the blue spherical cap, $\bp_1^T\bp_2 \leq \cos\theta$, otherwise, $\bp_1^T\bp_2 > \cos\theta$}
	\label{fig:valid_area}
	\vspace{-0.5em}
\end{figure}
\begin{prop}
	Let $\bp_1, \bp_2 \in \bbS^{d-1}$ be two points independently and uniformly sampled from the unit sphere in $d$-dimensional space. Each point is represented by a $d$-dimensional $\ell_2$ normalized vector. 
	The probability density of $\bp_1^T\bp_2$ satisfies 
	\begin{equation*}
	p(\bp_1^T\bp_2 = s) = 
	\left\{\begin{aligned}
	& \frac{(1-s^2)^{\frac{d-1}{2}-1}}{B(\frac{d-1}{2},\frac{1}{2})} && -1 \leq s \leq 1\\
	& 0 && \text{otherwise, }\\
	\end{aligned}
	\right.
	\label{eq:probability_density}
	\end{equation*}
	%
    where $B(a,b)$ is the beta function. 
	\label{prop:probability}
\end{prop}

\begin{proof}
Since $-1 \leq \bp_1^T\bp_2 \leq 1$, the second equation is obvious. To show the first equation, we calculate the cumulative distribution first. Here we only consider the case when $-1 \leq s <0$, and $0 \leq s <1$ can be shown similarly. Without loss of generality, we fix $\bp_1$ and also assume that $s = \cos\theta$, as shown in Figure~\ref{fig:valid_area}.
%
Since $\cos(\cdot)$ is a monotone function in $[0,\pi]$, $\bp_1^T\bp_2 \leq s$ if and only if $\bp_2$ is located on the blue spherical cap. Since $\bp_2$ is uniformly sampled from the sphere, the probability of $\bp_1^T\bp_2 \leq s$ is equal to the area of the blue spherical cap divided by the area of the whole sphere. The area of the $d-1$ dimensional spherical cap is 
\begin{equation*}
	S = \frac{1}{2}S_0r^{d-1}I_{(2rh-h^2)/r^2}(\frac{d-1}{2},\frac{1}{2}),
\end{equation*} 
where $S_0$ is the area of the whole sphere. $r = 1$ is the radius of the sphere. $h$ is the height of the spherical cap: $h = r + rcos(\theta) = r+sr$. $I_x(a,b)$ is the regularized incomplete beta function. Therefore, the cumulative distribution can be written as,
\begin{equation*}
		P(\bp_1^T\bp_2 \leq s) = \frac{1}{2}I_{1-s^2}(\frac{d-1}{2},\frac{1}{2}), \quad -1 \leq s < 0. 
\end{equation*}
  The probability density is the derivative of the cumulative distribution. 
\end{proof}

Figure \ref{fig:density} shows the probability distribution of the inner product (cosine similarity) of two points independently and uniformly sampled from the unit sphere\footnote{Note that since we assume the descriptors stay on the unit sphere ($\ell_2$ normalized), there is only a sign and constant difference between $\ell_2$ distance and cosine similarity, and the cosine similarity equals to the inner product.}. It shows that with high probability, the two independently and uniformly sampled points are close to orthogonal. %

\begin{figure}[t]
\begin{center}
	\includegraphics[width=0.8\linewidth]{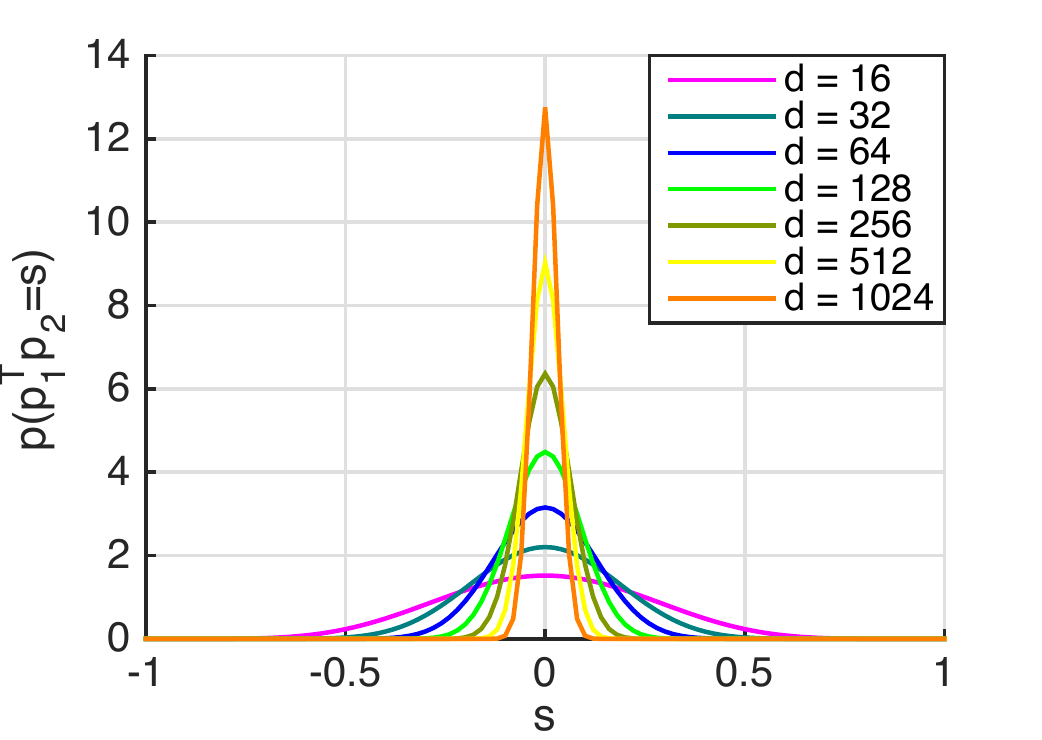}
\end{center}
\caption{Probability density of inner product of two points which are independently and uniformly sampled from the unit sphere in $d$-dimensional space. We can see that, in high dimensional space, most pairs are close to orthogonal.}
\label{fig:density}
\vspace{-0.5em}
\end{figure}

Based on the above observation, one might hope to make the distribution of the learned descriptors matches that of the uniform distribution. This is not practical in two ways. 1) How the learned descriptors distribute depends not only on the learned model, but also on the natural distribution of the image patches (not controllable). 2) It is technically difficult to match two distributions in practice. Instead, in this paper, we propose a regularization technique inspired by the theoretic properties of the uniform distribution on unit sphere. 
The regularization encourages the inner product of two randomly sampled non-matching descriptors matches that of two points independently and uniformly sampled from the unit sphere in its \emph{mean and second moment}.

The following proposition shows that for two points that are independently and uniformly sampled on the unit sphere, the mean and the second moment of their inner product are $0$ and $1 / d$, respectively.
\begin{prop}
	Let $\bp_1, \bp_2 \in \bbS^{d-1}$ be two points independently and uniformly sampled from the unit sphere. The mean and the second moment of $\bp_1^T\bp_2$ are
	\begin{equation*}
		\bbE(\bp_1^T\bp_2) = 0 \;\; \text{and} \;\; \bbE((\bp_1^T\bp_2)^2) = \frac{1}{d}.
	\end{equation*}
	\label{prop:mean_variance}
\end{prop}
\vspace{-2.5em}
\begin{proof}
	Due to symmetry, it's easy to show that $\bbE(\bp_1) = \bbE(\bp_2) = \mathbf{0}$, since $\bp_1$ and $\bp_2$ are independent,
	\begin{equation*}
		\bbE(\bp_1^T\bp_2) = \bbE(\bp_1^T)\bbE(\bp_2) = \mathbf{0}^T\mathbf{0} = 0.
	\end{equation*}
	Since both $\bp_1$ and $\bp_2$ are uniformly sampled from the unit sphere, when considering the second moment of the inner-product, we can fix one point and let the other point to be uniformly sampled. Without loss of generality, we choose $\bp_1 = [1,0,\ldots,0]^T$ and denote $\bp_2$ as $[p_{21},\ldots,p_{2d}]^T$, thus, 
	\begin{equation*}
		(\bp_1^T\bp_2)^2 = p_{21} \quad \text{and} \quad \bbE((\bp_1^T\bp_2)^2) = \bbE(p_{21}^2).
	\end{equation*}
	Due to the symmetry of the sphere, we can have $\bbE(p_{21}^2) = \bbE(p_{22}^2) = \ldots = \bbE(p_{2d}^2)$. Since $\bp_2$ is on the unit sphere, $\sum_{i = 1}^d p_{2i}^2 = 1$. Thus, $\bbE(p_{21}^2) = 1/d$. 
\end{proof}

\subsection{Global orthogonal regularization}
\label{subsec:global_regularization}

\begin{figure*}[h]
\begin{center}
	\includegraphics[width=0.85\linewidth]{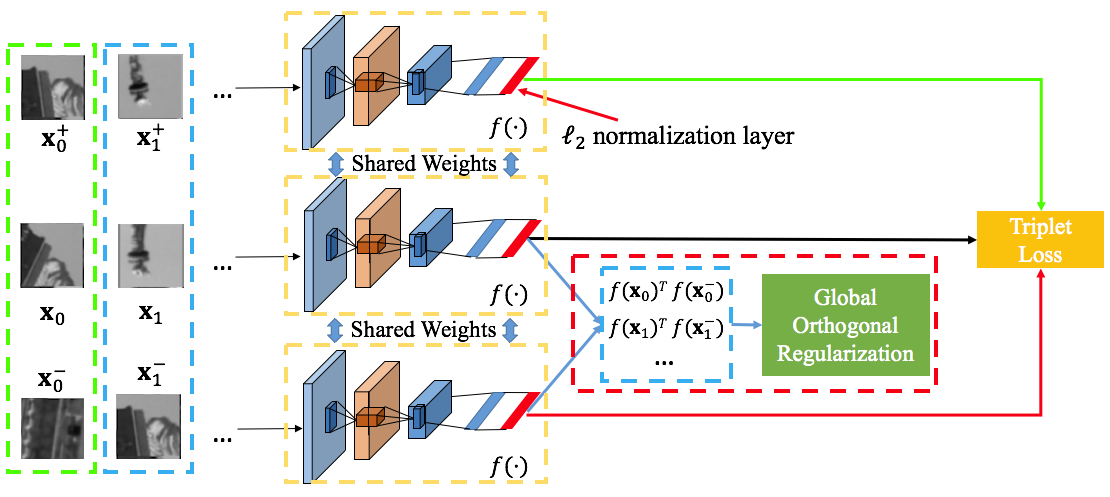}
\end{center}
\caption{Local feature descriptor training pipeline with triplet loss and the proposed global orthogonal regularization (GOR). GOR can also be used with the pairwise loss. In that case, there will be two branches of the network (known as the Siamese network) instead of three.}
\label{fig:training}
\vspace{-0.5em}
\end{figure*}

We propose a regularization which tries to match the mean and second moment shown in Proposition \ref{prop:mean_variance}. It encourages that the descriptors of random sampled non-matching pairs have similar statistical property as two points independently and uniformly sampled from the unit sphere.
We call this regularization Global Orthogonal Regularization (GOR). 
Following the notations in Section~\ref{sec:back}, given a set of $N$ random sampled non-matching patches $\{(\bx_i,\bx^-_i)\}_{i = 1}^{N}$, denote the descriptor function as $f(\cdot)$. The sample mean of the inner product of the descriptors of non-matching pairs is,
\begin{equation}
	M_1 = \frac{1}{N}\sum^N_{i = 1}f(\bx_i)^Tf(\bx^-_i).
\label{eq:empirical_mean}
\end{equation}
The sample second moment of the inner product is, 
\begin{equation}
\setlength\abovedisplayskip{5pt}
\setlength\belowdisplayskip{5pt}
	M_2({f(\bx)^Tf(\bx^-)}) = \frac{1}{N}\sum^N_{i = 1}(f(\bx_i)^Tf(\bx^-_i))^2.
\label{eq:empirical_variance}
\end{equation}
The Global Orthogonal Regularization (GOR) is defined as
\begin{equation}
\setlength\abovedisplayskip{5pt}
\setlength\belowdisplayskip{5pt}
	\ell_{\text{gor}} = M_1^2 + \max(0,M_2-\frac{1}{d}),
\label{eq:orthogonal_regularization}
\end{equation}
where $d$ is the dimension of the final output descriptor.


%
In (\ref{eq:orthogonal_regularization}), the first term tries to match the mean of the distributions and the second term tries to make the second moment close to $1/d$. 
In order to calculate the regularization term, one needs to consider all the non-matching pairs in the training set -- this is impractical. In practice we use a sampled batch to estimate its value.
The reason for using the hinge loss for the second term is that, in many batches, all the non-matching pairs are already very close to being orthogonal ($M_2 < 1/d$), and there is no need to force $M_2$ to be $1/d$. We have tried other loss functions for the second term. $\ell_1$ loss results in a similar performance, while $\ell_2$ leads to slight degradation.

The proposed regularization term can be used with any loss function. Denote the original training loss as $\ell_{(\cdot)}$, the final loss can be written as,
\begin{equation}
\setlength\abovedisplayskip{3pt}
\setlength\belowdisplayskip{3pt}
	\ell_{(\cdot)\_\text{gor}} = \ell_{(\cdot)} + \alpha \ell_{\text{gor}},
\label{eq:triplet_gor}
\end{equation}
where $\alpha$ is a tunable parameter.
In the experiment section, we test combining the global orthogonal regularization with contractive loss\footnote{For contractive loss (\ref{eq:pairwise}), we substitute the second term in (\ref{eq:pairwise}) with the proposed regularization, since both terms try to separate the non-matching pairs.} (\ref{eq:pairwise})~\cite{simo-serra_discriminative_2015}, triplet loss (\ref{eq:triplet})~\cite{balntas_learning_2016}, lifted structured similarity softmax loss (LSSS)~\cite{oh2016deep} and N-pair loss~\cite{sohn_improved_2016}. 

\subsection{Non-Euclidean distance}
\label{subsec:related}
So far, we assumed that the distance of the descriptors is based on the Euclidean distance.
There are several recent works that use a decision network instead of the Euclidean distance to calculate the similarity. 
Han \etal~\cite{han_matchnet:_2015} propose to use a Siamese network followed by a decision net. 
Zagoruyko and Komodakis~\cite{zagoruyko_learning_2015} develop a 2-stream networks, in which one stream focuses on the central area of the patch and the other focuses on the surrounding area of the patch. 
Kumar \etal \cite{g_learning_2015} propose a global loss and combine it with the 2-stream networks to achieve the state-of-the-art performance. The drawback of using a new type of distance rather than Euclidean distance is that efficient large-scale nearest neighbor search method such as locality sensitive hashing (LSH) \cite{pauleve2010locality} can no longer be used. In this paper, we focus on training local descriptor in the Euclidean space.

\section{Implementation}
In this section, we show our training framework based on triplet loss and the proposed global orthogonal regularization. The framework has three branches (as shown in Figure~\ref{fig:training}). The proposed global orthogonal regularization only considers two branches which process the non-matching pairs. 
Training pipeline of other losses can be achieved accordingly. For example, for the pairwise loss, we use a network with two branches (known as the Siamese network) instead of three.

Though our method is flexible in terms of the patch sizes, here we follow \cite{han_matchnet:_2015} to use patch size $64\times64$. Each branch in the triplet/Siamese network has the following structure: \{Conv(7,7,32) - MaxP(2,2) - Conv(6,6,64) - MaxP(2,2) - Conv(5,5,128) - MaxP(2,2) - FC(128) - $\ell_2$ Norm\}. 
Conv($n,m,c$) means convolutional layer with kernel size $(n,m)$ and output channel number $c$. MaxP($n,m$) is a max pooling layer with size $n \times n$ and stride $m$. FC($d$) is a fully connected layer with output dimension $d$. $\ell_2$ Norm is $\ell_2$ normalization layer to guarantee each descriptor has unit norm. 
All the convolution layers are followed by batch normalization~\cite{ioffe_batch_2015} and ReLU.  
Based on our implementation, when trained without the proposed global regularization, the above network structure achieves similar performance as the one proposed in \cite{balntas_learning_2016}. 
The motivation of the use of the above shallow network is for efficiency and avoiding overfitting~\cite{balntas_learning_2016,simo-serra_discriminative_2015}. We show the experiment of the proposed method over the large-scale patch descriptor benchmark in the next section. 

\begin{table*}[t]
\begin{center}
\begin{small}
\begin{tabular}{|c|c|c|cccccc|c|}
\hline
\multirow{3}{0.8cm}{\centering Loss Type}  & Training & & NotreDame & Liberty & NotreDame & Yosemite & Yosemite &Liberty &\multirow{3}{*}{Mean} \\ 
           & Test & &\multicolumn{2}{c}{Yosemite}  & \multicolumn{2}{c}{Liberty}& \multicolumn{2}{c|}{NotreDame} & \\ \cline{2-9}
           & Descriptor & Dim &\multicolumn{6}{c|}{ } & \\ \hline \hline
\multirow{2}{1.2cm}{\centering N/A} & SIFT~\cite{lowe_distinctive_2004} &128 & \multicolumn{2}{c}{27.29} & \multicolumn{2}{c}{29.84} & \multicolumn{2}{c|}{22.53} &26.55 \\
& VGG-Opt~\cite{simonyan_learning_2014} & 80 & 10.08 & 11.63& 11.42 & 14.58 & 7.22 & 6.17  &10.28 \\
\cline{1-2}
\multirow{4}{1.2cm}{\centering Pairwise} &DeepCompare$_{siam}$\cite{zagoruyko_learning_2015} & 256 & 15.89 & 19.91& 13.24 & 17.25 & 8.38 & 6.01  &13.45 \\
&DeepCompare$_{2str}$\cite{zagoruyko_learning_2015} & 512 & 13.02 & 13.24& 8.79 & 12.84 & 5.58 & 4.54  &9.67 \\
&DeepDesc\cite{simo-serra_discriminative_2015} & 128 & \multicolumn{2}{c}{16.19} & \multicolumn{2}{c}{8.82} & \multicolumn{2}{c|}{4.54} &9.85 \\
&\textbf{CL+GOR (Ours)} & 128 & 6.88 & 6.99 & 6.46 & 8.33 & 3.73 & 3.40 & 5.97 \\
\cline{1-2}
Global&TGLoss\cite{g_learning_2015} & 256 & 9.47 & 10.65& 9.91 & 13.45 & 5.43 & 3.91  &8.80 \\
\cline{1-2}
\multirow{4}{1.2cm}{\centering Triplet}&TFeat\cite{balntas_learning_2016} & 128 & 7.95 & 8.10& 7.64 & 9.88 & 3.83 & 3.39  &6.79 \\
&TFeat+AS\cite{balntas_learning_2016} & 128 & 7.08 & 7.82& 7.22 & 9.79 & 3.85 & 3.12  &6.47\\
&\textbf{TL+GOR (Ours)} & 128 & \textbf{4.94} & 5.74 & 5.47 & 7.13 & 2.58 & 2.28  &4.69 \\
&\textbf{TL+AS+GOR (Ours)} & 128 & 5.15 & \textbf{5.40}& \textbf{4.80} & \textbf{6.45} & \textbf{2.38} & \textbf{1.95}  &\textbf{4.36} \\

\cline{1-2}
\multirow{2}{1.2cm}{\centering Structured} & N-pair\cite{sohn_improved_2016} & 128 & 5.53 & 8.29 & \textbf{4.80} & 7.51 & 3.01 & 2.60  & 5.29 \\
&\textbf{N-pair+GOR (Ours)} & 128 & 5.16 & 7.43 & 5.03 & 7.10 & 2.81 & 2.34  &4.98 \\
 
\hline
\end{tabular}%
\end{small}
\end{center}
\caption{FPR95 (\%) of different methods on UBC patch dataset.  TL+AS+GOR achieves the lowest FPR95 rate.}
\label{table:FPR95}
\vspace{-0.5em}
\end{table*}

\section{Local Descriptor Result}
\label{sec:local}
\subsection{Dataset}
We first conduct experiments on the standard local patch descriptor benchmark, UBC patch dataset~\cite{brown_discriminative_2011}. The dataset contains three subsets, Yosemite, Notre Dame and Liberty. Each subset consists of more than 100k classes which include different image patches corresponding to the same 3D location obtained through a 3D reconstruction from different multi-view images. 
The total number of local image patches within each subset is more than 450k.
Each patch has a size of $64\times64$ and is sampled around the output of difference of Gaussian~(DOG)~\cite{lowe_distinctive_2004} detector. The scale and  orientation of the patch is normalized by the detector. Though with normalized scale and orientation, the patch dataset still contains great variations in view points, lighting, camera conditions \etc
We follow the evaluation protocol proposed in \cite{brown_discriminative_2011} to separate the whole dataset into six training-test combinations in which one subset is for training and the other for test. 

The metric used to evaluate different methods is false positive rate at 95\% true positive rate (FPR95), which is the standard metric in previous works~\cite{balntas_learning_2016,simo-serra_discriminative_2015,zagoruyko_learning_2015}. 
The test split of each subset contains 100k patch pairs in which 50\% are matching pairs and the other 50\% are non-matching pairs. The test pairs are predefined in \cite{brown_discriminative_2011}.

\subsection{Training setting and evaluation method}
For training, we randomly sample 1M triplets (for triplet network), or 1M matching pairs and 1M non-matching pairs (for Siamese network), for each training subset. 
No data augmentation or specially designed sampling is used. 
The training batch size is set to 128. We use SGD with momentum in the optimization. The learning rate starts at 0.1, with momentum 0.9. The learning rate is reduced after each epoch by a factor of 0.96. 
The trade-off parameter $\alpha$ in (\ref{eq:triplet_gor}) is set to 1 (we discuss the choice of $\alpha$ in Section \ref{subsec:local_result}). The loss function in (\ref{eq:orthogonal_regularization}) is hinge loss. The margin for matching pair in Siamese network ($\epsilon^+$ in (\ref{eq:pairwise})) is set to 0.7 and the margin of triplet network ($\epsilon$ in (\ref{eq:triplet})) is set to 0.5. All are estimated via empirical cross validation. 
Our implementation is based on TensorFlow~\cite{tensorflow2015-whitepaper}. The training of each epoch takes about 10 minutes on a Titan X GPU. All the networks are trained with 20 epochs, and they all converge before the end of training.

We compare our method with a large set of local feature descriptors which use Euclidean distance as similarity metric. 
The methods include: 1) hand-crafted descriptor (SIFT~\cite{lowe_distinctive_2004}), conventional machine learning based descriptor (VGG-Opt~\cite{simonyan_learning_2014}), 
2) deep learning based descriptors learned with pairwise loss (DeepCampare$_{siam}$~\cite{zagoruyko_learning_2015}, DeepCampare$_{2str}$\footnote{Subscript ``$2str$'' means central-surround network proposed in \cite{zagoruyko_learning_2015}.}~\cite{zagoruyko_learning_2015} and DeepDesc~\cite{simo-serra_discriminative_2015}), 
3) descriptors learned with triplet loss with and without anchor swap (TFeat+AS~\cite{balntas_learning_2016} and TFeat~\cite{balntas_learning_2016}),
4) descriptors learned with global loss (TGLoss~\cite{g_learning_2015}), 
and 
5) \xuworries{descriptors learned with structured loss (N-pair~\cite{sohn_improved_2016})}.

We combine the proposed Global Orthogonal Regularization with four commonly used losses mentioned in Section~\ref{sec:back}, namely, contractive loss, triplet loss, triplet loss with anchor swap and N-pair loss~\cite{sohn_improved_2016}.
Thus four variants of our method are used in evaluation: contractive loss with global orthogonal regularization (CL+GOR), triplet loss  with GOR (TL+GOR), triplet loss with anchor swap \cite{balntas_learning_2016} and GOR (TL+AS+GOR) and N-pair loss with GOR (N-pair+GOR).

\begin{figure}[t]
\begin{center}
\subfigure[Train: Notre Dame, $\qquad$ $\qquad$ Test: Liberty.]{\label{fig:notradame_liberty_roc}
	\includegraphics[width=0.48\linewidth]{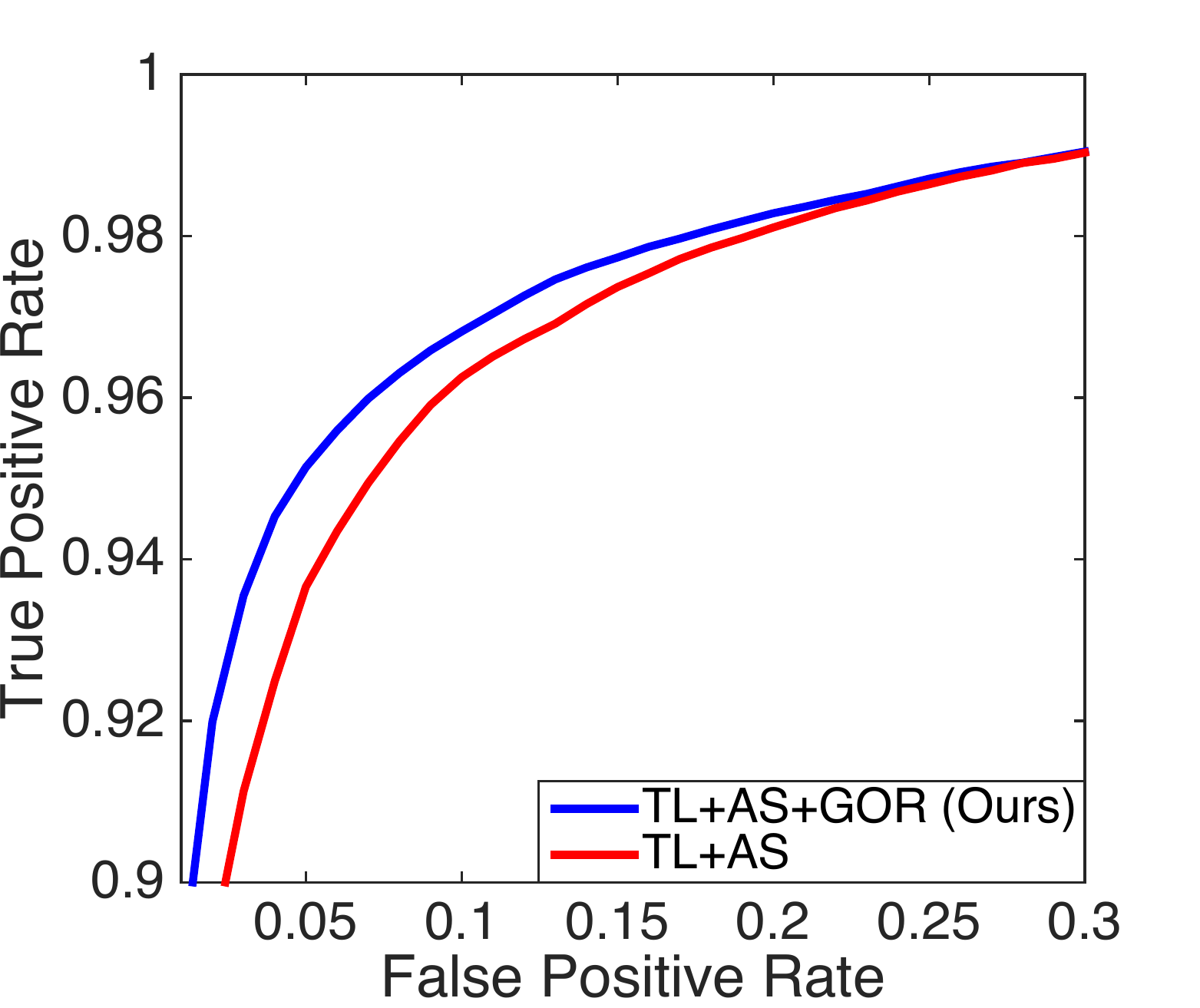}}
\subfigure[Train: Notre Dame, $\qquad$ $\qquad$ Test: Yosemite.]{\label{fig:notradame_yosemite_roc}
	\includegraphics[width=0.48\linewidth]{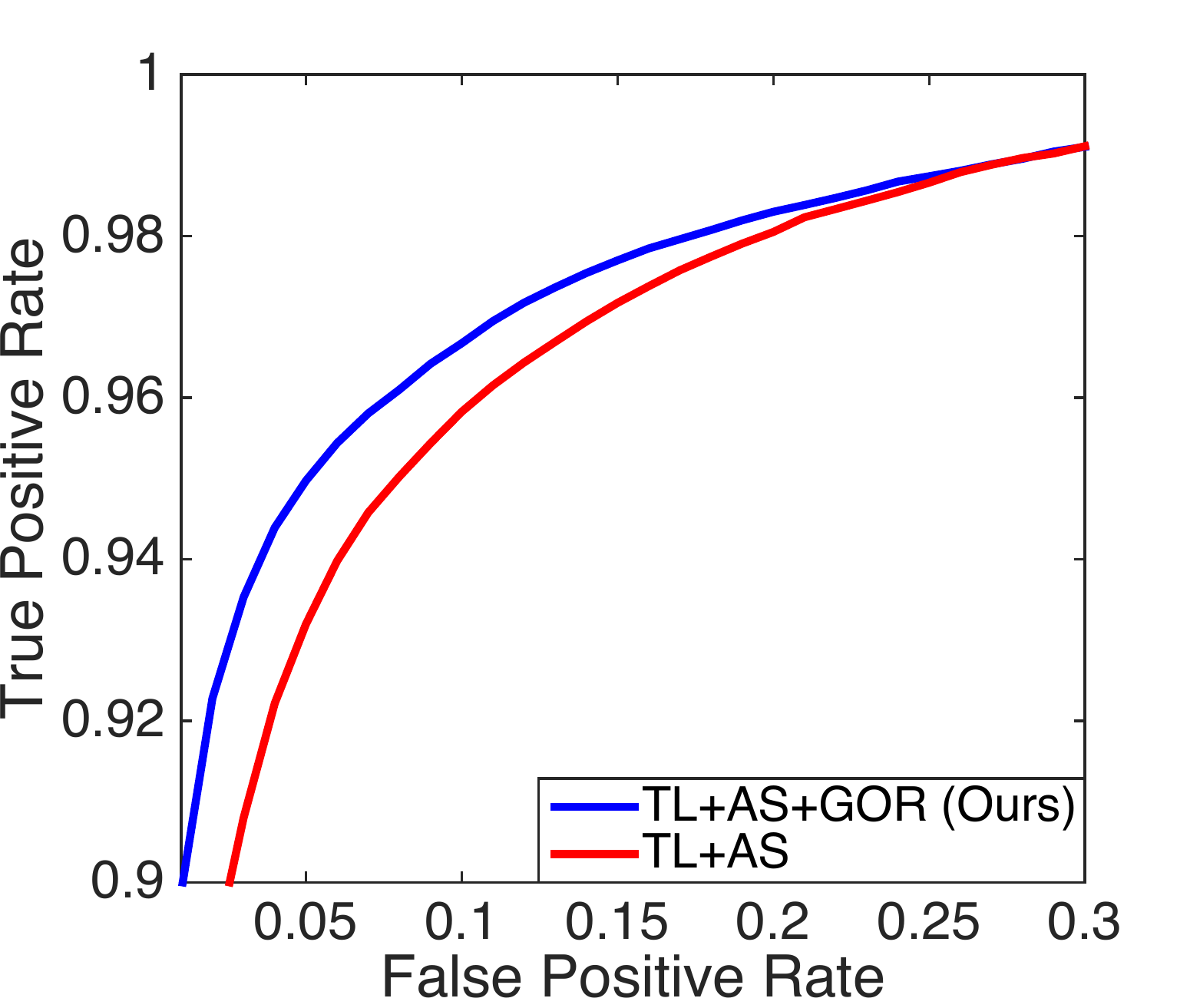}}
\end{center}
\caption{ROC curves for our method and baseline method trained on the Notre Dame subset and tested on the Liberty and Yosemite subsets.}
\label{fig:ROC}
\vspace{-1.0em}
\end{figure}

\subsection{Patch pair classification result}
\label{subsec:local_result}
\noindent\textbf{Classification error.} Table~\ref{table:FPR95} summarizes the performance of all the evaluated  Euclidean embedding methods on UBC patch dataset. We show FPR95 on each of the six training-test combinations and also the mean over all of them.

By simply applying the proposed global orthogonal regularization, almost all the baseline methods show performance gains. Specifically, among all the pairwise loss based methods, contractive loss with the proposed GOR~(CL+GOR) reduces the error of the previous best pairwise loss model~(DeepCampare$_{2str}$) from 9.67 to 5.97 with a relative deduction of 38.3\%.
Among all the triplet loss based methods, triplet loss with the proposed GOR (TL+GOR) reduces the error of its triplet loss baseline (TFeat) from 6.79 to 4.69. For the anchor swap version (TL+AS+GOR vs. TFeat+AS), the error reduces from 6.47 to 4.36. The relative deductions are 30.9\% and 32.6\%, respectively. \xuworries{For the structured loss, the error was reduced from 5.29 to 4.98. The improvement is not as significant because the N-pair loss already has the ability to force the random non-matching pairs to be orthogonal. The second moment of the non-matching pairs trained with N-pair loss is close to $2/d$, while that of the triplet loss is close to $50/d$. Overall, TL+AS+GOR achieves the lowest FPR95 rate.}

\xuworries{The improvement of our method can also be shown using other metrics such as the ROC curves. The ROC curves of TL+AS+GOR and TL+AS both are shown in Figure~\ref{fig:ROC}. Here, the training subset is Notre Dame, and the test subsets are Liberty (Figure~\ref{fig:notradame_liberty_roc}) and Yosemite (Figure~\ref{fig:notradame_yosemite_roc}. The result shows that the performance gain of the proposed regularization is universal at different false positive rates.}

\begin{figure}[t]
\subfigure[TL+AS (Baseline)]{\label{fig:histogram-baseline}
	\includegraphics[width=0.46\linewidth]{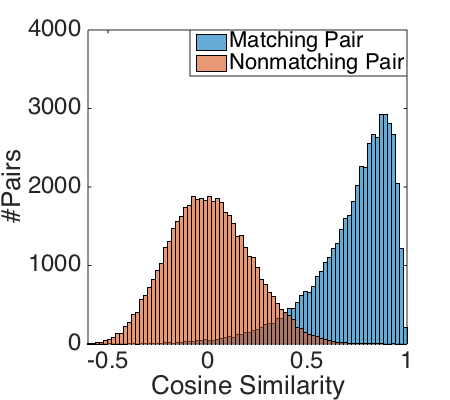}}
\subfigure[TL+AS+GOR (Ours)]{\label{fig:histogram-ours}
	\includegraphics[width=0.46\linewidth]{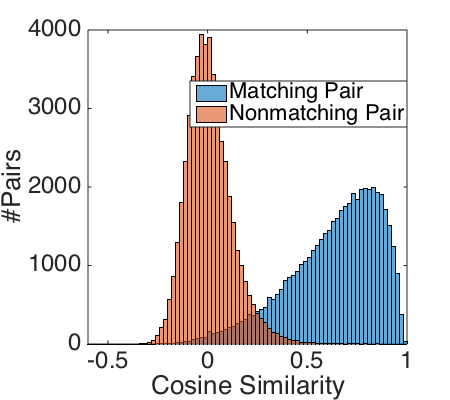}}
\hspace{1em}
\caption{Histogram of cosine similarity of matching pairs and non-matching pairs on ``Liberty''.  The model is trained on ``Notre Dame''. When trained with GOR, the non-matching pairs are more close to being orthogonal.}
\label{fig:histogram}
\end{figure}

\begin{figure}[t]
\begin{center}
\subfigure[FPR95(\%) with different values of $\alpha$ (\#dimension= 128).]{\label{fig:alpha}
	\includegraphics[width=0.45\linewidth]{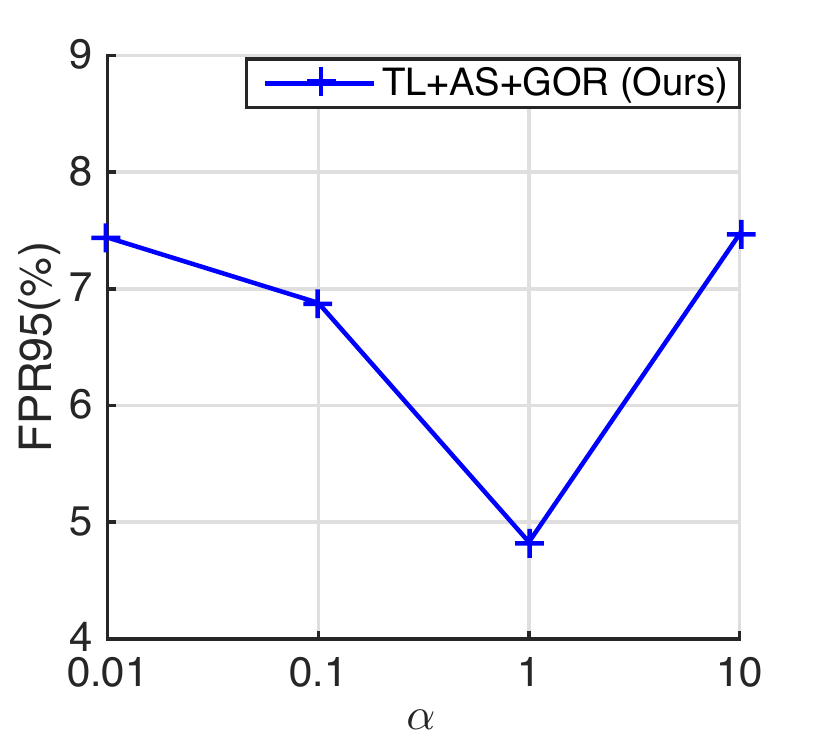}}
\subfigure[FPR95(\%) with different feature dimensions ($\alpha= 1$).]{\label{fig:embedding_num}
	\includegraphics[width=0.51\linewidth]{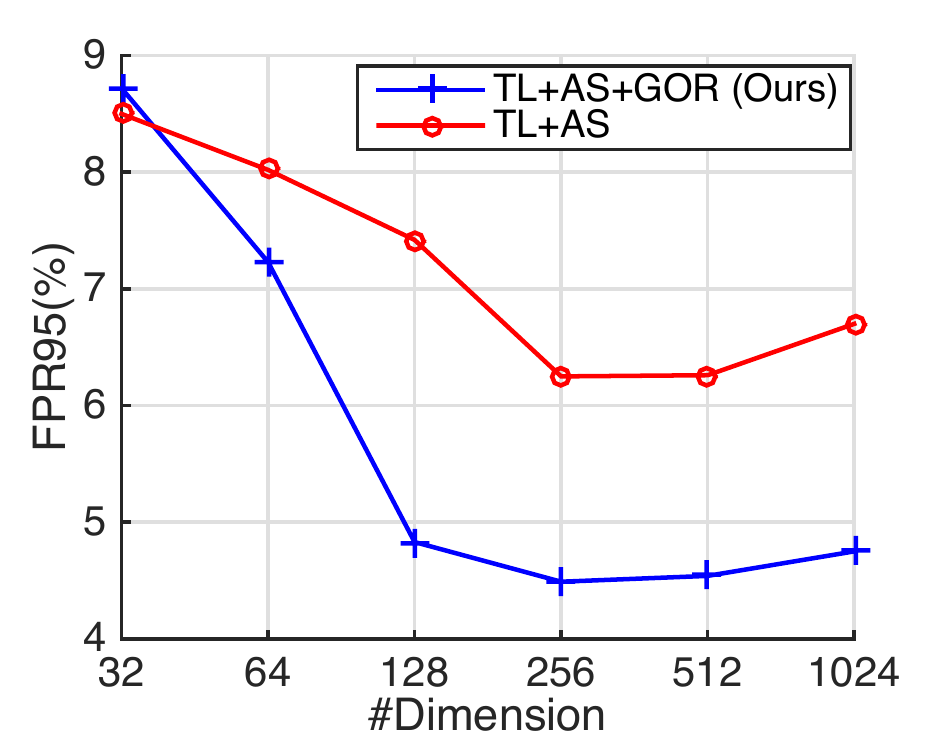}}
\end{center}
\caption{FPR95(\%) with different $\alpha$ and embedding dimensions. 
$\alpha$ trades off the regularization term and the triplet loss. Training set: Notre Dame, Test set: Liberty.}
\label{fig:sensitivity}
\vspace{-1.0em}
\end{figure}


\begin{figure*}[t]
	\begin{center}
	\subfigure[F1]{\label{fig:f1}
	\includegraphics[width=0.325\linewidth]{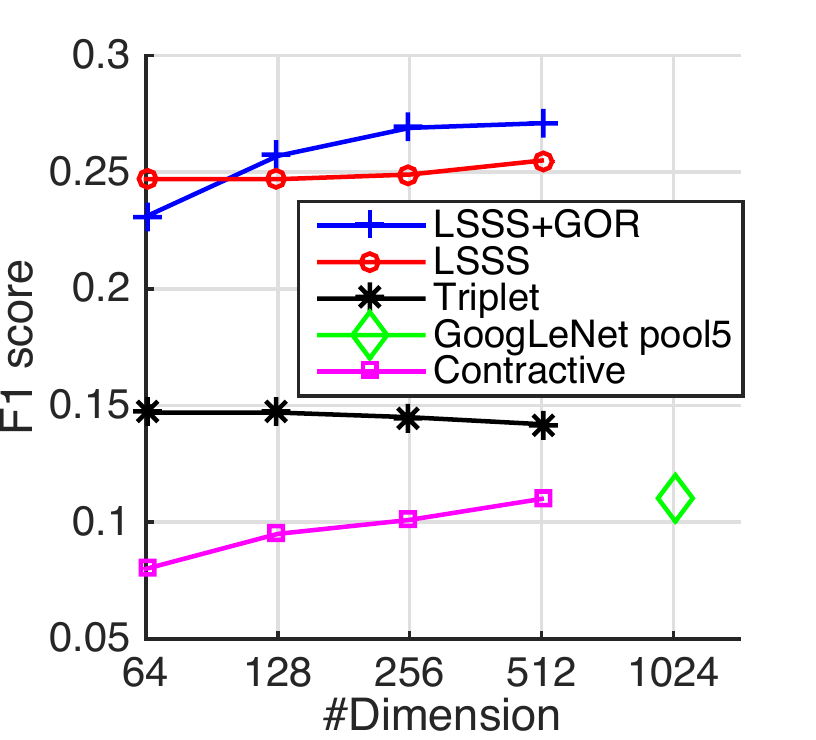}}
	\subfigure[NMI]{\label{fig:nmi}
	\includegraphics[width=0.32\linewidth]{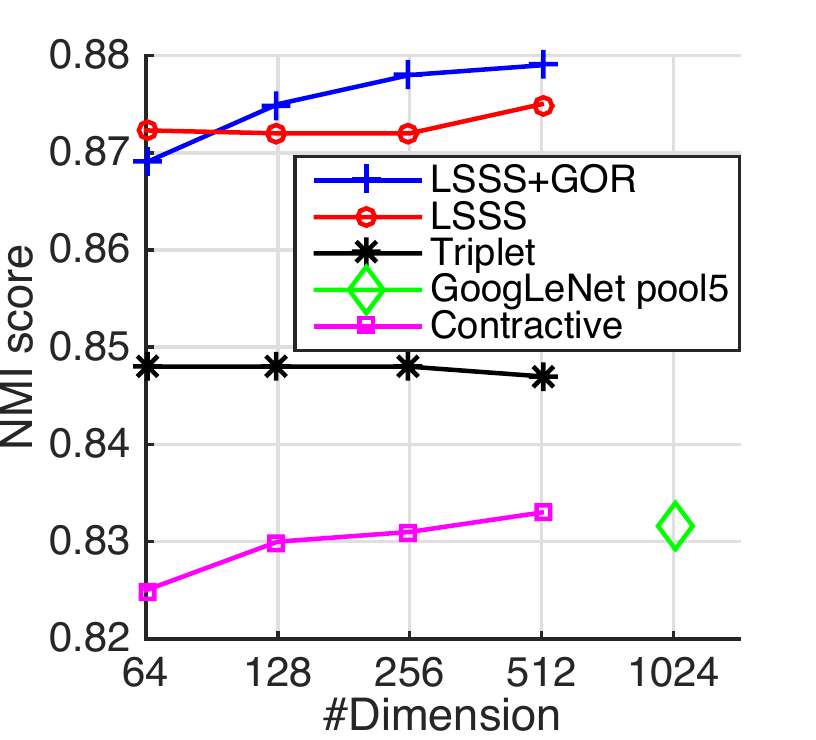}}
	\subfigure[Recall@K]{\label{fig:recall}
	\includegraphics[width=0.32\linewidth]{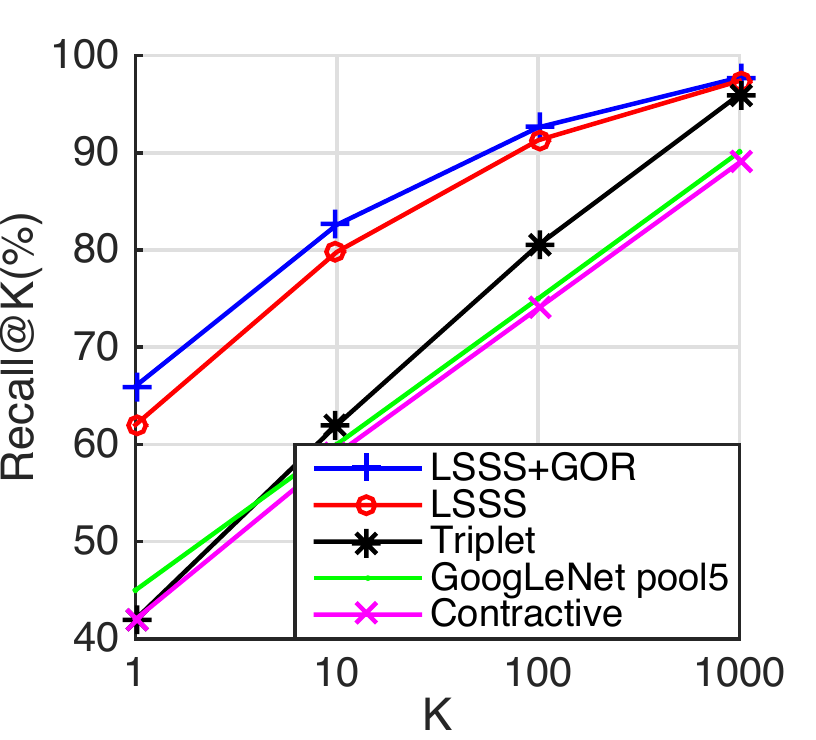}}
	\end{center}
	\caption{F1, NMI (for clustering) and Recall@K (for retrieval) scores for image-level descriptor learning using Stanford Online Product dataset.}
	\vspace{-0.5em}
\end{figure*}

\noindent\textbf{Similarity histogram.} 
To understand how the proposed GOR affects the distribution of the similarity, the histograms of cosine similarity of the matching pairs and non-matching pairs of the models trained with/without the proposed GOR on test set are shown in Figure~\ref{fig:histogram}. We use the same baseline method defined above. 
This figure shows the setting in which the training subset is Notre Dame and the test subset is Liberty, but the observation is also general for other training/test combinations.   

The histogram of the similarity of the matching and non-matching pairs of the baseline method is shown in Figure~\ref{fig:histogram-baseline}, while those of model trained with the proposed regularization is shown in Figure~\ref{fig:histogram-ours}. The histogram in blue is for matching pairs, while the histogram in orange is for non-matching pairs. The histogram of the similarity of non-matching pairs trained with GOR has a much sharper shape than that without the proposed regularization, which means when trained with GOR, non-matching pairs are more likely to be close to orthogonal. 
\xuworries{With the proposed GOR, the empirical error (overlapped area in Figure~\ref{fig:histogram-ours}) decreases by 15\% relatively in comparison with the baseline (Figure~\ref{fig:histogram-baseline}).}

\vspace{0.3em}
\noindent\textbf{Trade-off parameter.} $\alpha$ in (\ref{eq:triplet_gor}) controls the trade-off between the triplet loss and GOR. We use Notre Dame as training set and Liberty as test set and show the FPR95 of different models trained with different $\alpha$ values (from 0.01 to 10) in Figure~\ref{fig:alpha}. When $\alpha = 0$, (\ref{eq:triplet_gor}) becomes standard triplet loss. When $\alpha$ is large, the network will enforce the descriptor of all the non-matching pairs (including ``hard negatives'') to be close to orthogonal. 

\vspace{0.3em}


\vspace{0.3em}
\noindent\textbf{Embedding dimension.} We investigate how the proposed GOR affects the training of descriptors of different dimensionalities. We change the output node number in final fully-connected layer from $[32, 64, 128, 256, 512, 1024]$. The result is shown in Figure~\ref{fig:embedding_num}. 
The proposed GOR achieves significant performance gain when training a high-dimensional descriptor ($d \geq 64$). 
The low dimensional case ($d = 32$) does not work as well. One possible reason is that the descriptors of two non-matching patches are harder to be spread-out, and forcing non-matching patches to be orthogonal may lead to error. 
Finally, both our method and the baseline degrade for very high dimensions. We conjecture this is due to over-fitting. 
\xuworries{One may think that when $d$ is large, the network may not be able to force the second moment to a very small $1/d$. However, the proposed GOR is only a regularization not a hard constraint. And we can always make a trade-off by changing the value of $\alpha$ in (\ref{eq:triplet_gor}).}

\subsection{Descriptor extraction efficiency}
Since there are hundreds of patches in one image, the speed of descriptor extraction is also very important. The proposed GOR only affects the training stage, adding no additional cost in extraction pipeline. Based on our implementation on TensorFlow, when running a Titan X GPU, the extraction speed is about 10K patches per second, which is comparable to the conventional local descriptor extraction method like SIFT~\cite{lowe_distinctive_2004} and descriptor learning techniques using ``shallow'' structure such as TFeat and DeepDesc. 

\section{Extension to image-level embedding}
\label{sec:image}
Although GOR is proposed to learn local descriptors, the method can also be used in other applications where a feature embedding is learned. 
As an example, we show that it can also be used to improve the performance of image-level embedding. We compare our method to LSSS~\cite{oh2016deep}, which, as reviewed in Section \ref{subsec:improve}, outperforms triplet and pairwise losses.

\subsection{Dataset and evaluation metric}
The image level feature embedding experiment is conducted on Stanford Online Products dataset~\cite{oh2016deep}. Stanford Online Products dataset contains 120,053 product images crawled from eBay.com. There are a total of 22,634 products belonging to 12 categories. 
Each product is an individual class and has an average of 5.3 images. We strictly follow the same experiment setting in~\cite{oh2016deep}, that using 11,318 classes with a total of 59,551 images for training and another 11,316 classes with 60,502 images for test. The training and test splits have no overlap and are predefined in the dataset. We choose this dataset due to its realistic setting and rich variations within classes. 

As in \cite{oh2016deep}, we perform both clustering and retrieval tasks. For the clustering task, the F1 and NMI scores are used as the evaluation metrics~\cite{manning2008introduction}. F1 metric computes the harmonic mean of precision and recall. NMI metric equals to the mutual information divided by the average value of the entropy of clusters and the entropy of labels.
For retreival task, the performance is evaluated by Recall@K score as in \cite{oh2016deep}. For each query image, we first remove the query from the test set and then retrieve its K nearest neighbors from the test set. 
The recall of the test image is set to 1 if \emph{any} image in the same class with the query is retrieved and 0 otherwise. 

\subsection{Implementation details}
The proposed GOR is embedded with the lifted structured similarity softmax loss (LSSS), which is one of the best performing losses used in learning feature embedding. 
The network structure follows GoogLeNet~\cite{szegedy2015going} up to the ``pool5'' layer. The final descriptor is generated by a fully connected layer. All the convolutional layers are initialized from the network pre-trained on ImageNet ILSVRC dataset~\cite{russakovsky2015imagenet}.
All convolutional layers are fine-tuned with a learning rate that is 10 times smaller than that of the fully-connected layer. The batch size is set to 128 and the training iteration is set to 20,000. 

%
%
%

\subsection{Result}
\label{subsec:image_feature_result}
%


Figure~\ref{fig:f1} and Figure~\ref{fig:nmi} further show the F1 score and NMI score for the clustering task with different embedding sizes. 
By combining the proposed GOR with LSSS, our method shows better performance especially in high-dimensional cases ($d \geq 128$). The reason is discussed in Section \ref{subsec:local_result}. Figure~\ref{fig:recall} shows the Recall@K score for 512 dimensional descriptor. 

\xuworries{We also test the proposed regularization on small metric learning datasets such as Car196 and CUB-200-2011. The proposed regularization does not show clear improvement. One possible explanation is that the numbers of the classes in Car196~(196) and CUB-200-2011~(200) are much smaller than that of UBC~($>$100k) and Stanford online dataset~($>$22k). The assumption of uniform distribution for non-matching samples is not ideal for such situations. To understand this, one can imagine an extreme case of only two classes in a high dimensional space, putting them on opposite positions of the unit sphere (instead of orthogonal) is optimal.} 

\section{Conclusion}
\xuworries{We proposed a regularization technique named Global Orthogonal Regularization~(GOR) that makes the local feature descriptor more spread-out in the descriptor space. 
Inspired by the properties of uniform distribution, the regularization achieves the desired property by making the non-matching pairs close to orthogonal. We showed the  proposed regularization can be easily used to improve the performance of various feature embedding losses such as the pairwise and triplet losses.}

In the future, we plan to extend the proposed regularization technique to non-Euclidean distance. We also plan to apply our method to more general metric learning settings.
Our prototype implementation can be downloaded from \nolinkurl{https://github.com/ColumbiaDVMM/Spread-out_Local_Feature_Descriptor}.

\vspace{1em}

\noindent\textbf{Acknowledgement}
This material is based upon work supported by the United States Air Force Research Laboratory (AFRL) and the Defense Advanced Research Projects Agency (DARPA) under Contract No.~FA8750-16-C-0166. Any opinions, findings and conclusions or recommendations expressed in this material are solely the responsibility of the authors and does not necessarily represent the official views of AFRL, DARPA, or the U.S. Government.

\begin{small}
\bibliographystyle{ieee}
\bibliography{egbib}
\end{small}
\end{document}